\documentclass[11pt]{article}

\usepackage[margin=1in]{geometry}
\usepackage[T1]{fontenc}
\usepackage[utf8]{inputenc}
\usepackage{lmodern}
\usepackage{microtype}
\usepackage{amsmath,amssymb,amsthm,mathtools}
\usepackage{graphicx}
\usepackage{xcolor}
\usepackage{booktabs}
\usepackage{array}
\usepackage{enumitem}
\usepackage{algorithm}
\usepackage{algpseudocode}
\usepackage{hyperref}
\usepackage{tikz}
\usepackage{needspace}
\usetikzlibrary{arrows.meta,positioning,shapes.geometric,fit,backgrounds,calc}
\hypersetup{
  colorlinks=true,
  linkcolor=black,
  citecolor=black,
  urlcolor=blue,
  pdftitle={Probabilistic Language Tries: A Unified Framework for Compression, Decision Policies, and Execution Reuse},
  pdfauthor={Gregory Magarshak}
}

\title{Probabilistic Language Tries:\\
A Unified Framework for Compression,\\
Decision Policies, and Execution Reuse}

\author{Gregory Magarshak\\
\texttt{gmagarshak@faculty.ienyc.edu}}

\date{}

\newtheorem{definition}{Definition}
\newtheorem{theorem}{Theorem}
\newtheorem{proposition}{Proposition}
\newtheorem{corollary}{Corollary}
\newtheorem{lemma}{Lemma}
\newtheorem{remark}{Remark}

\begin{document}

\maketitle

\begin{abstract}
  We introduce \emph{probabilistic language tries} (PLTs), a unified representation
  that makes explicit the prefix structure implicitly defined by any generative model
  over sequences. By assigning to each outgoing edge the conditional probability of
  the corresponding token or action, a PLT simultaneously serves as (i) an optimal
  lossless compressor via frequency-weighted interval encoding---a generalization of
  arithmetic coding to model-conditioned distributions; (ii) a policy representation
  for sequential decision problems, including games, search, and robotic control;
  and (iii) a memoization index that lets repeated inference queries be answered
  by structured retrieval rather than full model execution.

  The central technical claim is a \emph{prior-guided caching theorem}: under a
  stationary generative distribution, a PLT-guided cache achieves lower expected
  inference cost than any empirical-frequency cache for all query counts below a
  threshold that grows with the strength of the prior.  Concretely, this converts
  an $O(n^2)$ transformer attention cost into an expected cost of
  $p_r \cdot O(\log N) + (1 - p_r) \cdot O(n^2)$, where $p_r$ is the
  prior-estimated reuse probability and $N$ is the artifact store size.

  We further introduce a \emph{hybrid compression architecture} that decomposes any
  dataset into a PLT-covered majority and a sparse residual store, achieving
  description lengths below the Shannon entropy of the empirical distribution
  whenever the generative model captures the true source structure.  This connects
  arithmetic coding with Kolmogorov-style program representations and with
  rate--distortion theory when approximate reconstruction is acceptable.

  We instantiate the framework across chess (MCTS-weighted opening tries),
  web search (workflow-weighted session tries), robotics, organizational workflows,
  and LLM inference systems, demonstrating that a single mathematical structure
  unifies compression, decision making, and computational reuse.
\end{abstract}

\tableofcontents
\newpage

\needspace{5\baselineskip}
\section{Introduction}

A generative model over sequences implicitly defines a probability distribution
over an enormous combinatorial space.  Modern large language models (LLMs) capture
this distribution in billions of parameters; game-playing agents capture it through
Monte Carlo Tree Search (MCTS) visit counts; search engines capture it through
click-through frequencies.  In each case, the structure of the distribution is
real and exploitable, yet it remains implicit and therefore difficult to use
directly for compression, caching, or explanation.

This paper proposes to make that structure \emph{explicit} through the
\emph{probabilistic language trie} (PLT): a rooted prefix tree in which each
outgoing edge carries the conditional probability of the corresponding symbol under
the underlying generative model.  Once made explicit, the same structure serves
three purposes simultaneously:

\begin{enumerate}[leftmargin=2em]
  \item \textbf{Compression.}  By assigning intervals proportional to conditional
        probabilities, the PLT induces a frequency-weighted interval code whose
        expected length equals the cross-entropy of the data under the model.
        Sequences well-predicted by the model receive short codes; surprising
        sequences receive long codes or are redirected to a sparse residual store.

  \item \textbf{Decision support.}  Any policy $\pi(s,a)$ over state--action pairs
        can be normalized into a conditional distribution and encoded as a PLT over
        action sequences.  The resulting structure simultaneously compresses
        experience, ranks actions, and organizes reusable strategic motifs.

  \item \textbf{Execution reuse.}  The PLT defines, without any empirical warmup,
        which inference queries are likely to recur.  By caching artifacts at
        high-probability nodes \emph{before} observing repeated requests, the system
        reduces expected inference cost from $O(n^2)$ to $O(\log N)$ for the
        majority of practical queries.
\end{enumerate}

\paragraph{Key distinction from prior work.}
Arithmetic coding~\cite{witten1987} and its neural
extensions~\cite{balle2018,townsend2019} already use learned distributions for
compression.  Semantic caches such as GPTCache~\cite{bang2023} and prefix
KV-caches~\cite{pope2023} reduce redundant computation empirically.
Our contribution is orthogonal to both: we show that the \emph{same} probabilistic
trie that defines the compression code also defines the optimal caching policy, and
that this policy strictly dominates empirical frequency caching during the initial
phase of a system's operation (Theorem~\ref{thm:prior_cache}).  The unification
of compression, policy representation, and caching under a single structure is
the main conceptual contribution.

\paragraph{Organization.}
Section~\ref{sec:plt} defines PLTs and the frequency-weighted interval encoding.
Section~\ref{sec:hybrid} introduces the hybrid architecture with a sparse residual
store and connects it to Shannon and Kolmogorov theories.
Section~\ref{sec:policy} extends the framework to policy-weighted decision languages.
Section~\ref{sec:execution} applies the framework to LLM execution and proves the
prior-guided caching theorem.
Section~\ref{sec:discussion} discusses implications and future directions.

\needspace{5\baselineskip}
\section{Probabilistic Language Tries and Interval Encoding}
\label{sec:plt}

\needspace{4\baselineskip}
\subsection{Definitions}

Let $V$ be a finite vocabulary (token set or action set) and let
$V^* = \bigcup_{n \geq 0} V^n$ denote the set of all finite sequences.
A \emph{generative model} $\mathcal{M}$ is a family of conditional distributions
$\{P_{\mathcal{M}}(\cdot \mid x) : x \in V^*\}$ with $\sum_{t \in V} P_{\mathcal{M}}(t \mid x) = 1$
for all $x$, together with a probability $P_{\mathcal{M}}(\$\mid x)$ of termination,
where $\$$ is a distinguished end-of-sequence symbol.

\begin{definition}[Probabilistic Language Trie]
\label{def:plt}
The \emph{probabilistic language trie} induced by $\mathcal{M}$ is the directed
rooted tree $\mathcal{T}(\mathcal{M})$ whose nodes are the prefixes in $V^*$ and
whose outgoing edges from node $x$ are labeled by tokens $t \in V \cup \{\$\}$
with weights $P_{\mathcal{M}}(t \mid x)$.  The weight function satisfies
\[
  \sum_{t \in V \cup \{\$\}} P_{\mathcal{M}}(t \mid x) = 1 \quad \forall x \in V^*.
\]
\end{definition}

The probability of a complete sequence $s = (t_1, \ldots, t_n)$ is the product of
edge weights along the path from the root:
\[
  P_{\mathcal{M}}(s) = P_{\mathcal{M}}(\$ \mid s) \prod_{i=1}^n P_{\mathcal{M}}(t_i \mid t_1,\ldots,t_{i-1}).
\]

\begin{remark}
  For a transformer language model, $P_{\mathcal{M}}(t \mid x)$ is the softmax output
  at position $|x|$ conditioned on the prefix $x$.  For an MCTS agent,
  $P_{\mathcal{M}}(a \mid s) = N(s,a) / \sum_{a'} N(s,a')$ where $N(s,a)$ is the
  visit count of action $a$ from state $s$.  The PLT formalism encompasses both.
\end{remark}

\needspace{4\baselineskip}
\subsection{Frequency-Weighted Interval Encoding}

We now construct a bijective map from $V^*$ into the unit interval $[0,1)$ whose
structure mirrors the PLT.  This generalizes standard arithmetic coding to
model-conditioned distributions.

\textbf{Base case.}  Assign the root node the interval $I_\varnothing = [0,1)$.

\textbf{Recursive step.}  Given that node $x = (t_1, \ldots, t_k)$ has interval
$I_x = [a_x, b_x)$ of width $|I_x| = b_x - a_x$, define a cumulative ordering of
tokens by any fixed bijection $\sigma: V \to \{1,\ldots,|V|\}$ and let
\[
  C_t(x) = \sum_{\sigma(u) < \sigma(t)} P_{\mathcal{M}}(u \mid x).
\]
Then the child interval for token $t$ is
\[
  I_{x \cdot t} = \bigl[a_x + |I_x| \cdot C_t(x),\;
                         a_x + |I_x| \cdot (C_t(x) + P_{\mathcal{M}}(t \mid x))\bigr),
\]
which satisfies $|I_{x \cdot t}| = |I_x| \cdot P_{\mathcal{M}}(t \mid x)$.

\begin{proposition}[Width identity]
\label{prop:width}
For any sequence $s = (t_1, \ldots, t_n)$,
\[
  |I_s| = \prod_{i=1}^n P_{\mathcal{M}}(t_i \mid t_1,\ldots,t_{i-1}).
\]
The full sequence probability including termination satisfies
$P_{\mathcal{M}}(s) = P_{\mathcal{M}}(\$ \mid s) \cdot |I_s|$.
Any real number $z \in I_s$ encodes $s$ in
$L(s) = \lceil -\log_2 |I_s| \rceil + 1$ bits.
\end{proposition}

\begin{proof}
By induction on $n$.  Base case $n=1$: $|I_{t_1}| = 1 \cdot P_{\mathcal{M}}(t_1 \mid \varnothing)$.  Inductive step: $|I_{x \cdot t}| = |I_x| \cdot P_{\mathcal{M}}(t \mid x)$, and by hypothesis $|I_x| = \prod_{i=1}^{k} P_{\mathcal{M}}(t_i \mid t_1,\ldots,t_{i-1})$, giving the product formula.  The code-length bound follows from the standard result that any interval of width $w$ can be encoded in $\lceil -\log_2 w \rceil + 1$ bits~\cite{witten1987}.
\end{proof}

\begin{corollary}[Expected code length]
\label{cor:entropy}
Let $\mathcal{D}$ be a distribution over $V^*$.  Then
\[
  \mathbb{E}_{s \sim \mathcal{D}}[L(s)] \leq H(\mathcal{D},\mathcal{M}) + 2,
\]
where $H(\mathcal{D},\mathcal{M}) = -\mathbb{E}_{s \sim \mathcal{D}}[\log_2 |I_s|]$
is the cross-entropy of $\mathcal{D}$ under $\mathcal{M}$ (using the prefix
probabilities, not including the termination factor).
When $\mathcal{D} = \mathcal{M}$, this equals $H(\mathcal{M}) + 2$, matching the
Shannon lower bound up to two bits (the constant-2 overhead is inherent to
integer-length codes; see~\cite{witten1987}, Theorem~1).
\end{corollary}

\begin{proof}
By Proposition~\ref{prop:width}, $|I_s| = \prod_{i} P_{\mathcal{M}}(t_i \mid t_1,\ldots,t_{i-1})$,
so
\[
  -\log_2 |I_s| = \textstyle\sum_i -\log_2 P_{\mathcal{M}}(t_i \mid t_1,\ldots,t_{i-1}).
\]
Since $L(s) \leq {-\log_2 |I_s|} + 2$, taking expectations:
$\mathbb{E}[L(s)] \leq H(\mathcal{D},\mathcal{M}) + 2$. \qed
\end{proof}

\needspace{4\baselineskip}
\subsection{Probability-Proportional Bijection and the Trie Metric}
\label{sec:bijection}

The map $\phi: V^* \to [0,1)$ sending $s$ to any point in $I_s$ is a
\emph{probability-proportional bijection}: high-probability sequences occupy large
subintervals and therefore require few bits; low-probability sequences occupy small
subintervals.  This differs fundamentally from fixed-radix
encodings such as base-256 or base-26, where each symbol receives an equal-width
subdivision regardless of its probability.

\begin{remark}[Ordering within $[0,1)$ is arbitrary]
\label{rem:ordering}
The bijection $\sigma: V \to \{1,\ldots,|V|\}$ determines the \emph{left-to-right
order} of child intervals within each node's subdivision.  By Proposition~\ref{prop:width},
the \emph{width} $|I_s| = P_{\mathcal{M}}(s)$ depends only on the model probabilities,
not on $\sigma$.  Consequently, the code length $L(s)$ and the compression optimality
of Corollary~\ref{cor:entropy} hold for \emph{every} choice of $\sigma$.
Any total ordering of $V$ produces an equally valid and equally optimal encoder.
\end{remark}

\begin{remark}[Locality is not preserved---and need not be]
\label{rem:locality}
The bijection $\phi$ is \emph{not} a space-filling curve in the sense of
Hilbert or Peano curves: it does not preserve locality.  Two sequences that are
semantically similar but share no common prefix will in general map to unrelated
parts of $[0,1)$.  This is intentional and correct.  Locality preservation would
require sacrificing probability-proportionality, and with it the near-optimal
compression guarantee of Corollary~\ref{cor:entropy}.

The natural notion of ``closeness'' for sequences in a PLT is not proximity
in $[0,1)$ but the \emph{trie metric}: the length of the longest common prefix.
Formally, define
\[
  d_{\mathcal{T}}(s, s') = -\log_2 P_{\mathcal{M}}(s \wedge s'),
\]
where $s \wedge s'$ denotes the longest common prefix of $s$ and $s'$, and
$P_{\mathcal{M}}(x) = \prod_{i=1}^{|x|} P_{\mathcal{M}}(t_i \mid t_1,\ldots,t_{i-1})$
is the probability of prefix $x$, with the convention $P_\mathcal{M}(\varnothing)=1$
so $d_\mathcal{T}(s,s)=0$.  Two sequences are close under $d_{\mathcal{T}}$
when they share a high-probability prefix---exactly the condition under which their
model-conditioned continuations are similar and cached artifacts transfer.
This is the metric used throughout Sections~\ref{sec:hybrid}--\ref{sec:hierarchical}
for approximation quality and residual analysis.
\end{remark}

\begin{remark}[$d_\mathcal{T}$ is a pseudometric]
\label{rem:metric}
$d_\mathcal{T}$ satisfies: (i) $d_\mathcal{T}(s,s) = 0$ for all $s$
(since $s \wedge s = s$ and $-\log_2 P_\mathcal{M}(s) \geq 0$ with equality only
if $P_\mathcal{M}(s)=1$, which does not hold for non-trivial sequences unless all
probability mass is on a single path); (ii) symmetry: $s \wedge s' = s' \wedge s$;
(iii) \emph{ultrametric} triangle inequality: $d_\mathcal{T}(s, s'') \leq \max(d_\mathcal{T}(s,s'),\, d_\mathcal{T}(s',s''))$, since the longest common prefix of $s$ and $s''$ is at least as long as the shorter of the longest common prefixes of $(s,s')$ and $(s',s'')$.
The ultrametric property is stronger than the ordinary triangle inequality and
reflects the tree structure: any two paths through a node are at most as far apart
as their distance to that node.
\end{remark}

Figure~\ref{fig:interval} illustrates a two-level interval subdivision showing
the recursive structure, with the trie alongside the interval for clarity.

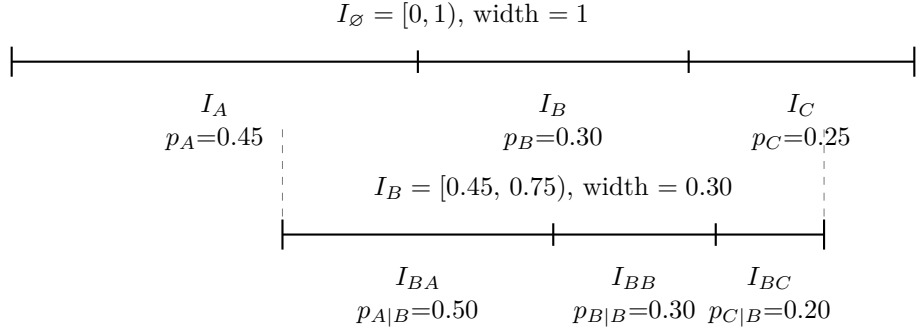
\begin{figure}[t]
\centering
\begin{tikzpicture}[
  seg/.style={draw=black, thick},
  lbl/.style={font=\small},
  node/.style={circle, draw=black, thick, minimum size=6mm, font=\small},
  edge_label/.style={font=\scriptsize, midway}
]
\node[lbl, above] at (6, 3.8) {$I_\varnothing = [0,1)$, width $= 1$};
\draw[seg] (0,3.5) -- (12,3.5);
\draw[seg] (0,3.7)  -- (0,3.3);
\draw[seg] (12,3.7) -- (12,3.3);

\draw[seg] (5.4, 3.65) -- (5.4, 3.35);
\draw[seg] (9.0, 3.65) -- (9.0, 3.35);

\node[lbl, below] at (2.7, 3.2) {$I_A$};
\node[lbl, below] at (7.2, 3.2) {$I_B$};
\node[lbl, below] at (10.5, 3.2) {$I_C$};
\node[lbl, below] at (2.7, 2.75) {$p_A{=}0.45$};
\node[lbl, below] at (7.2, 2.75) {$p_B{=}0.30$};
\node[lbl, below] at (10.5, 2.75) {$p_C{=}0.25$};

\node[lbl, above] at (7.2, 1.5) {$I_B = [0.45,\,0.75)$, width $= 0.30$};
\draw[seg] (3.6, 1.2) -- (10.8, 1.2);
\draw[seg] (3.6, 1.4) -- (3.6, 1.0);
\draw[seg] (10.8,1.4) -- (10.8,1.0);

\pgfmathsetmacro{\bsplitone}{3.6 + 7.2*0.50}  
\pgfmathsetmacro{\bsplittwo}{3.6 + 7.2*0.80}  
\draw[seg] (7.2, 1.35) -- (7.2, 1.05);
\draw[seg] (9.36, 1.35) -- (9.36, 1.05);

\node[lbl, below] at (5.4, 0.9)  {$I_{BA}$};
\node[lbl, below] at (8.28,0.9)  {$I_{BB}$};
\node[lbl, below] at (10.08,0.9) {$I_{BC}$};
\node[lbl, below] at (5.4, 0.45)  {$p_{A|B}{=}0.50$};
\node[lbl, below] at (8.28,0.45)  {$p_{B|B}{=}0.30$};
\node[lbl, below] at (10.08,0.45) {$p_{C|B}{=}0.20$};

\draw[dashed, gray] (3.6,2.6)  -- (3.6,1.4);
\draw[dashed, gray] (10.8,2.6) -- (10.8,1.4);

\end{tikzpicture}
\caption{Two-level frequency-weighted interval subdivision.  \emph{Top row:}
  the root interval $[0,1)$ is partitioned into $I_A$, $I_B$, $I_C$ with widths
  proportional to $p_A{=}0.45$, $p_B{=}0.30$, $p_C{=}0.25$.
  \emph{Bottom row:} the interval $I_B$ is recursively partitioned into
  $I_{BA}$, $I_{BB}$, $I_{BC}$ with widths proportional to the conditional
  probabilities $P(\cdot \mid B)$.
  The width of $I_{BA}$ equals $0.30 \times 0.50 = 0.15 = P_{\mathcal{M}}(BA)$,
  confirming Proposition~\ref{prop:width}.
  The left-to-right ordering of siblings within each row is determined by an
  arbitrary fixed bijection $\sigma$ and does not affect code lengths
  (Remark~\ref{rem:ordering}).  Locality in $[0,1)$ is \emph{not} preserved:
  $I_{BA}$ and $I_A$ are adjacent in the interval but unrelated in the trie.}
\label{fig:interval}
\end{figure}

\needspace{4\baselineskip}
\subsection{Encoding and Decoding Algorithms}

\begin{algorithm}[t]
\caption{PLT interval encoding}
\label{alg:encode}
\begin{algorithmic}[1]
\Require Sequence $s = (t_1,\ldots,t_n)$; model $\mathcal{M}$
\Ensure Interval $I_s = [a,b)$ and code length $L$
\State $[a,b) \gets [0,1)$; $x \gets \varnothing$
\For{$i = 1$ to $n$}
  \State Query $\mathcal{M}$: obtain $\{P_{\mathcal{M}}(t \mid x)\}_{t \in V}$
  \State Compute $C_{t_i}(x) \gets \sum_{\sigma(u)<\sigma(t_i)} P_{\mathcal{M}}(u \mid x)$
  \State $w \gets b - a$
  \State $b \gets a + w \cdot (C_{t_i}(x) + P_{\mathcal{M}}(t_i \mid x))$
  \State $a \gets a + w \cdot C_{t_i}(x)$
  \State $x \gets (x, t_i)$
\EndFor
\State \Return $[a,b)$, $\lceil \log_2(b-a)^{-1} \rceil + 1$
\end{algorithmic}
\end{algorithm}

\begin{algorithm}[t]
\caption{PLT interval decoding}
\label{alg:decode}
\begin{algorithmic}[1]
\Require Real number $z \in [0,1)$; model $\mathcal{M}$; length $n$
\Ensure Sequence $s$
\State $[a,b) \gets [0,1)$; $x \gets \varnothing$; $s \gets ()$
\For{$i = 1$ to $n$}
  \State Query $\mathcal{M}$: obtain $\{P_{\mathcal{M}}(t\mid x)\}_{t \in V}$
  \State Find $t$ such that $z \in [a + (b-a)C_t(x),\; a + (b-a)(C_t(x)+P_{\mathcal{M}}(t\mid x)))$
  \State Append $t$ to $s$; update $[a,b)$ to child interval; $x \gets (x,t)$
\EndFor
\State \Return $s$
\end{algorithmic}
\end{algorithm}

\textbf{Complexity.}  Both algorithms require $O(n)$ calls to $\mathcal{M}$ and
$O(n \cdot |V|)$ arithmetic operations.  When prefix distributions are cached at
frequently visited trie nodes, the effective cost for repeated prefixes is
significantly lower.

\needspace{5\baselineskip}
\section{Hybrid Compression Architecture}
\label{sec:hybrid}

A PLT compresses sequences that are well-predicted by $\mathcal{M}$.  Real datasets
inevitably contain rare or surprising sequences for which $P_{\mathcal{M}}(s)$ is
small and the code length $-\log_2 P_{\mathcal{M}}(s)$ is large.  We handle these
with a \emph{sparse residual store}.

\needspace{4\baselineskip}
\subsection{Trie-Covered and Residual Decomposition}

\begin{definition}[Hybrid compression architecture]
\label{def:hybrid}
Let $\mathcal{D} = \{s_1,\ldots,s_m\}$ be a dataset and let $\tau > 0$ be a
length threshold.  Partition $\mathcal{D}$ as
\[
  \mathcal{D} = C_T \cup C_R,
  \quad C_T = \{s \in \mathcal{D} : L_{\mathcal{M}}(s) \leq \tau\},
  \quad C_R = \mathcal{D} \setminus C_T,
\]
where $L_{\mathcal{M}}(s) = \lceil -\log_2 P_{\mathcal{M}}(s) \rceil + 1$.
The \emph{hybrid description length} is
\[
  L(\mathcal{D}) = L(\mathcal{M}) + \sum_{s \in C_T} L_{\mathcal{M}}(s) + L_{\mathrm{res}}(C_R),
\]
where $L(\mathcal{M})$ is the cost of transmitting $\mathcal{M}$ and
$L_{\mathrm{res}}(C_R)$ is the cost of storing the residuals (e.g., via a
B-tree or separate lossless coder).
\end{definition}

\needspace{4\baselineskip}
\subsection{Escape Transitions}

To handle residuals within the encoding stream, we augment $V$ with an escape
symbol $E$ and define a modified distribution:
\[
  P'(t \mid x) =
  \begin{cases}
    (1 - \epsilon) \cdot P_{\mathcal{M}}(t \mid x) & t \in V, \\
    \epsilon & t = E,
  \end{cases}
\]
for a small escape probability $\epsilon > 0$.  When the encoder encounters a
sequence $s \in C_R$, it emits $E$ at the point of divergence and stores
the remainder as a correction in the residual store.

\needspace{4\baselineskip}
\subsection{Connection to Shannon Entropy}

Shannon's noiseless coding theorem~\cite{shannon1948} establishes that the expected
code length satisfies
\[
  \mathbb{E}[L(X)] \geq H(X)
\]
for any lossless code, with equality achieved by an arithmetic code matched to the
true source distribution.  Our hybrid architecture achieves this bound when
$\mathcal{M} = \mathcal{D}$, i.e., when the model exactly matches the source:

\begin{proposition}[Optimality under matched model]
\label{prop:optimal}
If $P_{\mathcal{M}} = P_{\mathcal{D}}$ (the model matches the data distribution),
then $\mathbb{E}_{s \sim \mathcal{D}}[L_{\mathcal{M}}(s)] \leq H(\mathcal{D}) + 2$.
\end{proposition}

\begin{proof}
When $P_{\mathcal{M}} = P_{\mathcal{D}}$, the cross-entropy equals the entropy:
$H(\mathcal{D}, \mathcal{M}) = -\mathbb{E}_{s \sim \mathcal{D}}[\log_2 P_{\mathcal{M}}(s)]
= -\mathbb{E}_{s \sim \mathcal{D}}[\log_2 P_{\mathcal{D}}(s)] = H(\mathcal{D})$.
The result then follows immediately from Corollary~\ref{cor:entropy}. \qed
\end{proof}

When $\mathcal{M} \neq \mathcal{D}$, the average code length is the cross-entropy
$H(\mathcal{D}, \mathcal{M}) = H(\mathcal{D}) + D_{\mathrm{KL}}(\mathcal{D} \| \mathcal{M})$,
which exceeds the entropy by the KL divergence between source and model.

\needspace{4\baselineskip}
\subsection{Kolmogorov-Style Interpretation}

Kolmogorov complexity $K(s)$ of a sequence $s$ is the length of the shortest
program that generates $s$ on a universal Turing machine~\cite{kolmogorov1965}:
\[
  K(s) = \min_{p : U(p) = s} |p|.
\]
Though uncomputable, it provides a conceptual benchmark.  In our setting, the
PLT induced by a large generative model $\mathcal{M}$ acts as a compact
\emph{effective program} that generates most of the observed corpus.  The hybrid
scheme therefore operationalizes a computable approximation to Kolmogorov
compression:

\begin{definition}[Practical Kolmogorov optimality]
\label{def:kolmogorov}
Fix a description scheme in which any dataset $\mathcal{D}$ can be encoded as a
pair $(\mathcal{M}, C_R)$ with total description length
$L(\mathcal{M}) + |C_R| \cdot \bar{L}_{\mathrm{res}}$, where $L(\mathcal{M})$ is
measured in bits and $\bar{L}_{\mathrm{res}}$ is the average per-residual code
length under a universal lossless code (e.g., Lempel--Ziv~\cite{ziv1977}).
A hybrid architecture $(\mathcal{M}, C_R)$ is \emph{practically Kolmogorov-optimal}
for dataset $\mathcal{D}$ with respect to model class $\mathfrak{M}$ if
\[
  L(\mathcal{M}) + |C_R| \cdot \bar{L}_{\mathrm{res}}
  = \min_{\mathcal{M}' \in \mathfrak{M}} \bigl[
      L(\mathcal{M}') + |\mathcal{D} \setminus C_T(\mathcal{M}')| \cdot \bar{L}_{\mathrm{res}}
    \bigr],
\]
where $C_T(\mathcal{M}')$ denotes the trie-covered set under model $\mathcal{M}'$
at threshold $\tau$.  That is, no other model in $\mathfrak{M}$ achieves a shorter
total description of $\mathcal{D}$ under this scheme.
\end{definition}

In practice, when a strong model satisfies $|C_R|/|\mathcal{D}| \approx \epsilon$
for small $\epsilon$, the description length approaches the entropy of the
trie-covered majority $C_T$ plus a small overhead for the residuals.

\needspace{4\baselineskip}
\subsection{Approximate Compression and Rate--Distortion}

When exact reconstruction is not required, the encoder may map each $s \in C_R$ to
its nearest representative $\tilde{s} \in C_T$ on the trie manifold.  The
relevant theoretical framework is the rate--distortion function~\cite{berger1971}:
\[
  R(D) = \min_{P(\tilde{s}\mid s) :\, \mathbb{E}[d(s,\tilde{s})] \leq D} I(S;\tilde{S}),
\]
where $d(\cdot,\cdot)$ is a task-appropriate distortion measure.  The natural
distortion measure induced by the PLT is the \emph{trie metric}
(Remark~\ref{rem:locality}):
\[
  d(s, \tilde{s}) = -\log_2 P_{\mathcal{M}}(s \wedge \tilde{s}),
\]
where $s \wedge \tilde{s}$ is the longest common prefix of $s$ and $\tilde{s}$.
This measures how far the two sequences diverge within the trie: if they share a
long, high-probability common prefix, distortion is low; if they diverge early,
distortion is high.  Importantly, this is a genuine pairwise measure on $(s, \tilde{s})$,
not a property of $\tilde{s}$ alone.  Under this distortion, the optimal mapping
$s \mapsto \tilde{s}$ selects the cached sequence sharing the longest common
prefix with $s$:
\[
  \tilde{s}(s) = \arg\max_{s' \in C_T} |s \wedge s'|
  \quad \bigl(= \arg\min_{s' \in C_T} d(s, s')\bigr).
\]

This is the sense in which the hybrid architecture achieves rates below the lossless
entropy floor under a lossy constraint: by accepting bounded distortion
$D = -\log_2 P_{\mathcal{M}}(s \wedge \tilde{s})$, the rate--distortion function
$R(D)$ allows description lengths strictly below $H(\mathcal{D})$, consistent with
Shannon's source coding theorem for lossy coding~\cite{berger1971}.
The PLT-covered set $C_T$ acts as a learned codebook: probability mass is
concentrated where data is dense, enabling rates that would be impossible under the
lossless constraint $D = 0$.

\begin{remark}[Systems implications]
  The same trade-off applies beyond data compression: blockchain-style consensus,
  exact-match caching, and strict source-of-truth architectures often spend
  disproportionate resources on the last fraction of exactness.  The hybrid
  framework suggests a principled way to quantify and bound the cost of that
  exactness.
\end{remark}

\needspace{5\baselineskip}
\section{Policy-Weighted Languages for Decision Systems}
\label{sec:policy}

The PLT formalism extends naturally from passive data compression to
\emph{active} sequential decision making.  The key observation is that a
policy $\pi$ defines a conditional distribution over actions, which induces a PLT
over action sequences by the same construction as Definition~\ref{def:plt}.

\needspace{4\baselineskip}
\subsection{Decision Sequences as Languages}

Let $\mathcal{S}$ be a state space and $\mathcal{A}(s)$ be the set of actions
available in state $s$.  A policy $\pi: \mathcal{S} \times \mathcal{A} \to \mathbb{R}_{\geq 0}$
assigns non-negative weights to state--action pairs.  Normalize to obtain:
\[
  P_\pi(a \mid s) = \frac{\pi(s,a)}{\sum_{a' \in \mathcal{A}(s)} \pi(s,a')}.
\]
A trajectory $\tau = (s_0, a_0, s_1, a_1, \ldots, s_n)$ is a sequence of
state--action pairs, and its probability under $\pi$ is
$P_\pi(\tau) = \prod_{i} P_\pi(a_i \mid s_i) \cdot P(s_{i+1} \mid s_i, a_i)$.
The PLT over action sequences (marginalizing over states) compresses experience,
ranks prefixes by policy value, and organizes reusable trajectory fragments.

\needspace{4\baselineskip}
\subsection{Games: MCTS-Weighted Opening Tries}
\label{sec:games}

In two-player perfect-information games such as chess or Go, a high-quality policy
can be derived from MCTS visit counts.  Let $N(s,m)$ denote the number of times
move $m$ was explored from position $s$ during MCTS.  Define
\[
  P_{\mathrm{MCTS}}(m \mid s) = \frac{N(s,m)}{\sum_{m'} N(s,m')}.
\]
The PLT over move sequences under $P_{\mathrm{MCTS}}$ has the following properties:

\begin{remark}[Game-trie properties]
\label{rem:game}
Under $P_{\mathrm{MCTS}}$:
\begin{enumerate}[leftmargin=2em]
  \item Common opening sequences (e.g., the Ruy Lopez, Sicilian Defence) have high
        $P_{\mathrm{MCTS}}$ by definition of visit-count normalization, so by
        Proposition~\ref{prop:width} they occupy large subintervals and receive
        short codes.
  \item Blunders and rarely played continuations have low visit counts, hence low
        probability, hence small intervals; they fall into the residual store $C_R$
        when $L_{\mathcal{M}}(s) > \tau$.
  \item The trie organizes moves into a hierarchy of strategic signposts:
        prefixes correspond to named openings, mid-game motifs, and endgame
        structures.
  \item The code length $L(s) = \lceil -\log_2 P_{\mathrm{MCTS}}(s) \rceil + 1$
        of a game record measures the \emph{unpredictability} of the game from the
        engine's perspective; it is a formal novelty score.
\end{enumerate}
Items 1 and 2 follow immediately from Proposition~\ref{prop:width} and
Definition~\ref{def:hybrid}; items 3 and 4 are interpretive consequences.
\end{remark}

This interpretation is distinct from previous uses of tries in game trees
(e.g., opening books): here the trie is a \emph{compressor} of game records,
not merely a lookup table.  Several non-obvious consequences follow.

\textbf{Novelty detection.}  A move that falls into the residual store ---
whose code length under the MCTS trie exceeds threshold $\tau$ --- is by definition
a \emph{novelty}: a line not well-explored by the engine.  The PLT provides
an automatic novelty detector with a threshold tied directly to escape probability
$\epsilon$, requiring no separate classification step.

\textbf{Style as distribution.}  Two players with similar styles produce game
records with high overlap in the trie.  The KL divergence between their
respective PLTs quantifies stylistic similarity, suggesting a
compression-based approach to player profiling and matchmaking that does not
require hand-crafted features.

\textbf{Unifying opening books and tablebases.}  Endgame tablebases store the
optimal outcome for every position below a piece-count threshold.  In PLT terms,
these are exact residuals: the endgame language has low entropy (outcomes are
forced), so $C_T$ is small and $C_R$ dominates.  The PLT framework therefore
unifies opening books (high-probability prefixes in $C_T$) and tablebases
(exact residuals in $C_R$) under a single hybrid architecture.

\textbf{Transfer across game variants.}  The MCTS trie for standard chess and
Chess960 share a large common subtree (all positions reachable after move 10 are
identically distributed under most opening preparation).  The PLT provides a
principled mechanism for knowledge transfer: the shared subtree is reused directly,
and only the diverging branches require fresh MCTS computation.

\needspace{4\baselineskip}
\subsection{Search Engines: Workflow-Weighted Session Tries}
\label{sec:search}

A user session can be modeled as a sequence of queries and page visits:
\[
  \sigma = (q_1, p_1, q_2, p_2, \ldots, q_k, p_k),
\]
where $q_i \in \mathcal{Q}$ is a query and $p_i \in \mathcal{P}$ is a page.
Traditional IR systems such as PageRank assign authority scores to individual
pages~\cite{page1999}; the session probability $P(\sigma)$ is not modeled.

The PLT framework suggests a fundamentally different objective: rather than ranking
documents, rank \emph{workflows} by their probability under a session model.
Let $\pi_{\mathrm{session}}(q_i, p_i \mid q_1,p_1,\ldots,q_{i-1},p_{i-1})$
estimate the probability that a user takes action $(q_i, p_i)$ given prior context.
The resulting PLT assigns large intervals to common productive workflows (e.g.,
a user searching for a flight and completing a booking) and small intervals to
abandoned or unusual sessions.

\begin{definition}[Workflow language]
The \emph{workflow language} of a search system is the PLT over session sequences
induced by the session policy $\pi_{\mathrm{session}}$.  A workflow optimizer
selects actions at each step to maximize the probability of the current prefix
under the workflow language.
\end{definition}

This extends PageRank from a stationary distribution over nodes to a
\emph{stationary distribution over paths}, enabling optimization of complete task
sequences rather than isolated actions.

The shift has several concrete implications beyond ranking.

\textbf{Workflow compression.}  A session-log database can be compressed using
the workflow PLT exactly as a text corpus is compressed using a language PLT.
Common productive sessions occupy large intervals; abandoned or anomalous sessions
occupy small intervals or fall into residuals.  The compression ratio directly
measures how well the system's model captures user behavior.

\textbf{Proactive prefetching.}  At any point in a session, the current prefix
determines a node in the workflow trie whose child distribution predicts the most
likely next actions.  The system can prefetch the top-$K$ continuations by prior
probability, reducing latency for the most likely next steps.  By
Theorem~\ref{thm:prior_cache}, this prior-guided prefetching dominates reactive
caching during the early phase of a new session --- precisely when personalization
data is scarcest.

\textbf{Task-completion as a language.}  Traditional IR optimizes relevance at the
document level.  The workflow PLT optimizes \emph{task completion probability}:
the probability that the current session prefix reaches a high-value terminal state
(purchase completed, answer found, form submitted).  This reframes IR as a
language modeling problem over session sequences, where perplexity measures how
well the system predicts complete workflows, not individual clicks.

\textbf{Anomalous session detection.}  Sessions whose code length under the
workflow PLT exceeds threshold $\tau$ are residuals: they deviate significantly
from the modeled distribution.  Such sessions are natural candidates for fraud
detection and security monitoring without requiring a separate anomaly model.

\needspace{4\baselineskip}
\subsection{Robotics, Agents, and Organizations}

The framework extends identically to physical and organizational domains.

\textbf{Robot control.}  Action sequences are motor commands; the policy is a
learned controller.  Common manipulation trajectories (e.g., pick-and-place in
a known environment) receive short codes and can be memoized as reusable motion
primitives.  Novel situations --- an object in an unexpected position,
an obstacle in a familiar path --- produce long codes and trigger recomputation
or fallback to slower deliberative planning.  The code length of a trajectory
is thus a real-time measure of task novelty, providing a principled trigger
for switching between cached and recomputed behavior.

\textbf{Network routing.}  Traffic flows are sequences of forwarding decisions;
the policy is a traffic model learned from observed flows.  The PLT compresses
routing tables (frequent paths receive shorter codes), predicts congestion
(high-entropy nodes correspond to unpredictable traffic), and enables
proactive route caching for common source-destination pairs.

\textbf{Organizational workflows.}  Business processes are sequences of steps
(approvals, document submissions, notifications); the policy is learned from
historical execution logs.  The PLT identifies the most probable process variants
and flags deviations as residuals.  In audit contexts, the residuals are precisely
the exceptional cases warranting human review, with a principled threshold for
what counts as ``exceptional.''  Process mining~\cite{van2016} currently uses
heuristic deviation measures; the PLT provides an information-theoretic replacement.

\textbf{LLM agent loops.}  Tool invocations in an agent system are action
sequences; the policy is the agent's planning distribution.  The PLT predicts
which tool-call sequences are likely, enabling proactive artifact retrieval
(Section~\ref{sec:execution}).  Moreover, the trie structure makes the agent's
planning transparent: each step in the tool-call sequence is annotated with
its prior probability, and deviations from the most probable plan are
immediately visible as low-probability transitions.

In each domain, the incremental update rule
\[
  P_{t+1}(a \mid s) \propto (1-\alpha) P_t(a \mid s) + \alpha \hat{P}(a \mid s)
\]
allows the trie to adapt online as new observations accumulate, improving
compression and decision quality over time.

\needspace{4\baselineskip}
\subsection{Unified Compression View}

Across all domains, the PLT plays three simultaneous roles:

\begin{enumerate}[leftmargin=2em]
  \item \textbf{Compressor of experience}: high-probability trajectories receive
        short codes.
  \item \textbf{Policy map}: the transition probabilities encode learned action
        preferences.
  \item \textbf{Structural index}: the trie organizes experience into a
        hierarchy of reusable prefixes (openings, workflows, maneuvers).
\end{enumerate}

The key insight is that these three roles are \emph{not independent features}
but are \emph{derived from a single probability measure} on sequence space.
Any improvement to the model $\mathcal{M}$ simultaneously improves all three.

\needspace{5\baselineskip}
\section{Artifact Memoization and Compression of Model Execution}
\label{sec:execution}

We now apply the PLT framework recursively to the \emph{execution} of
generative models themselves.  The central insight is that the same probability
measure that defines the compression code also defines the optimal caching policy:
high-probability nodes in the execution trie correspond to computations worth
memoizing, and the model's own prior identifies them \emph{before} any requests
have been observed.

\needspace{4\baselineskip}
\subsection{Artifacts and Content-Addressable Storage}

\begin{definition}[Artifact]
An \emph{artifact} is a deterministic output $a = f(i)$ produced by a function
$f$ (a model, tool, or sub-computation) on input $i$.  Artifacts are stored
under a content address $h = H(f, i)$, where $H$ is a cryptographic hash function.
\end{definition}

Artifacts include generated text, images, code, reasoning traces, plans, and
intermediate tool results.  The key property is \emph{determinism}: given the
same $(f, i)$, the same $a$ is produced, so $h$ uniquely identifies the artifact.
This is not a restriction in practice: for stochastic models one fixes the random
seed as part of $i$, or stores the output distribution rather than a single sample.

\needspace{4\baselineskip}
\subsection{The Execution Language}

The sequence of tool or model invocations in an agentic system forms an
\emph{execution trace}:
\[
  e = ((f_1, i_1, a_1), (f_2, i_2, a_2), \ldots, (f_n, i_n, a_n)).
\]
Projecting onto invocation prefixes $(f_1,i_1), (f_2,i_2), \ldots$ yields a
language over the alphabet of (function, input) pairs.  The PLT over this language,
induced by the agent's planning distribution $P_{\mathcal{M}}$, assigns probabilities
to future invocations and organizes execution histories into a reuse-aware structure.

Concretely, two agents solving related tasks will share a long common prefix in
the execution trie --- the shared setup, retrieval, and reasoning steps --- before
diverging at the point where their specific goals differ.  The PLT makes this
shared prefix explicit and memoizable.

\needspace{4\baselineskip}
\subsection{Artifact Reuse as Compression of Computation}

Let $C_c$ be the cost of computing an artifact and $C_l$ be the cost of retrieving
a cached artifact.  For a transformer with context length $n$:
\[
  C_c = O(n^2), \qquad C_l = O(\log N),
\]
where $N$ is the number of stored artifacts.  The ratio $C_c/C_l = O(n^2/\log N)$
grows rapidly with context length, making caching increasingly valuable at scale.

The expected cost of a hybrid system with reuse probability $p_r$ is:
\[
  \mathbb{E}[C] = p_r \cdot O(\log N) + (1-p_r) \cdot O(n^2).
\]
As the artifact store grows and $p_r \to p^* = \sum_{j=1}^K p_j$
(the cumulative probability of the top-$K$ inputs), the system's amortized
inference cost approaches $O(\log N)$ for the fraction $p^*$ of queries.

\needspace{4\baselineskip}
\subsection{Prior-Guided vs.\ Empirical-Frequency Caching}
\label{sec:prior_cache}

Standard caches (LRU, LFU, semantic similarity~\cite{bang2023}) populate based on
\emph{observed} request frequencies.  They require a warmup phase before their
contents reflect the true request distribution.  We now make this cost precise
and show that PLT-guided caching eliminates it.

Let inputs be drawn i.i.d.\ from $P_{\mathcal{M}}$ over a support of size $M$,
with probabilities ranked $p_1 \geq p_2 \geq \cdots \geq p_M > 0$.
The \emph{optimal cache} is $\mathcal{C}^* = \{i_1,\ldots,i_K\}$, achieving
steady-state hit rate $p^* = \sum_{j=1}^K p_j$.
Define the \emph{gap at the cache boundary}:
\[
  \Delta = p_K - p_{K+1}.
\]
This gap governs how quickly an empirical cache can correctly identify $\mathcal{C}^*$.

\begin{definition}[Prior-guided cache]
\label{def:prior_cache}
A \emph{prior-guided cache} of size $K$ initializes with $\mathcal{C}^*$
(the top-$K$ inputs by prior probability) and holds it for all $T \geq 1$.
Its per-request cost is $C^{\mathrm{prior}} = (1-p^*)C_c + p^* C_l$ from
the first request onward.
\end{definition}

\begin{lemma}[LFU ranking threshold]
\label{lem:lfu_rank}
Let $\hat{p}_j(T) = n_j(T)/T$ denote the empirical frequency of input $i_j$
after $T$ i.i.d.\ requests.  For any pair $(j,l)$ with $j \leq K < l$
(so $p_j \geq p_K > p_{K+1} \geq p_l$ and $p_j - p_l \geq \Delta$),
\[
  P\!\left(\hat{p}_j(T) \leq \hat{p}_l(T)\right)
  \leq \exp\!\left(-T\Delta^2 / 2\right).
\]
\emph{Proof.}
Let $X_k = \mathbf{1}[r_k = i_j] - \mathbf{1}[r_k = i_l]$ for the $k$-th request $r_k$.
Each $X_k$ is i.i.d.\ with mean $\mu = p_j - p_l \geq \Delta$ and $|X_k| \leq 1$.
Then $\hat{p}_j(T) - \hat{p}_l(T) = \frac{1}{T}\sum_{k=1}^T X_k$.
By Hoeffding's inequality for bounded zero-mean variables (shifting by $\mu$):
$P(\hat{p}_j(T) - \hat{p}_l(T) \leq 0) = P\!\left(\frac{1}{T}\sum_k (X_k - \mu) \leq -\mu\right)
\leq \exp(-2T\mu^2/4) = \exp(-T\mu^2/2) \leq \exp(-T\Delta^2/2),$
where the factor $4$ in the denominator accounts for the range $[-1,1]$ of each $X_k - \mu$.

By a union bound over all $K(M-K)$ boundary pairs, the probability that LFU
has an incorrect ranking at time $T$ satisfies
\[
  P\!\left(\mathcal{C}_T^{\mathrm{LFU}} \neq \mathcal{C}^*\right)
  \leq K(M-K)\exp(-T\Delta^2/2).
\]
For this to be at most $\delta$, one needs
$T \geq T_{\mathrm{rank}}(\delta) = \frac{2}{\Delta^2}\ln\frac{K(M-K)}{\delta}$. \qed
\end{lemma}

\begin{lemma}[LFU swap completion]
\label{lem:lfu_swap}
Even after frequency estimates are correctly ranked, LFU must complete up to $K$
cache swaps before $\mathcal{C}^*$ is fully installed: each item of $\mathcal{C}^*$
must be requested at least once.  Let $T_{\mathrm{swap}}$ be the first time all $K$
items have been requested at least once.  Then:
\[
  \mathbb{E}[T_{\mathrm{swap}}] = \sum_{j=1}^{K} \frac{1}{p_j} \leq \frac{K}{p_K}.
\]
Moreover, for any $T \leq K/(2p_K)$,
\[
  P(T_{\mathrm{swap}} > T) \geq 1 - \frac{T p_K}{K} \geq \frac{1}{2}.
\]
\emph{Proof.}
The expected waiting time for item $i_j$ to appear is $1/p_j$ (geometric distribution
with parameter $p_j$).  The $K$ waiting times are not independent, but since
$p_j \leq p_1$ for all $j$, the sum $T_{\mathrm{swap}} \geq \max_j W_j$ where
$W_j \sim \mathrm{Geom}(p_j)$.  For the expectation: by the coupon-collector
structure, $\mathbb{E}[T_{\mathrm{swap}}] = \sum_{j=1}^K 1/p_j \leq K/p_K$.
For the lower tail: by Markov's inequality applied to $T_{\mathrm{swap}}$,
$P(T_{\mathrm{swap}} \leq T) \leq T/\mathbb{E}[T_{\mathrm{swap}}] \leq T p_K/K$.
Hence $P(T_{\mathrm{swap}} > T) \geq 1 - T p_K/K$, which is $\geq 1/2$ when
$T \leq K/(2p_K)$. \qed
\end{lemma}

\begin{theorem}[Prior-guided caching advantage]
\label{thm:prior_cache}
Let requests be i.i.d.\ from $P_{\mathcal{M}}$ over support of size $M$, with
$\Delta = p_K - p_{K+1} > 0$.
Let $\rho = C_c - C_l > 0$.  Define
\[
  T_0(\delta) = \min\!\left(\frac{2\ln(K(M-K)/\delta)}{\Delta^2},\;
                             \frac{K}{2p_K}\right).
\]
Then for all $T \leq T_0(\delta)$,
\[
  \mathbb{E}[C^{\mathrm{LFU}}(T)] - C^{\mathrm{prior}}(T) \geq \frac{1}{2}\,\Delta\,\rho\,
  \min\!\left(\delta,\,1 - \frac{T p_K}{K}\right) > 0.
\]
In particular, taking $\delta = 1/2$ gives a gap of at least
$\frac{1}{4}\Delta\rho > 0$ for all $T \leq T_0(1/2)$.
As $T \to \infty$, both strategies converge to the same steady-state cost.
\end{theorem}

\begin{proof}
\textbf{Prior-guided cost} is constant by Definition~\ref{def:prior_cache}:
$C^{\mathrm{prior}} = (1-p^*)C_c + p^* C_l$ for all $T \geq 1$.

\textbf{LFU has not converged.}
Let $B_T = \{T_{\mathrm{swap}} > T\}$ (the event that not all $K$ target items have
been requested at least once by time $T$).  By Lemma~\ref{lem:lfu_swap},
$P(B_T) \geq 1 - Tp_K/K$ for all $T$.  When $B_T$ occurs, LFU cannot have
installed $\mathcal{C}^*$ regardless of ranking quality: at least one item of
$\mathcal{C}^*$ has never been requested and therefore cannot be in the cache.

\textbf{Cost gap on $B_T$.}
On $B_T$, there exists $i_l \in \mathcal{C}^* \setminus \mathcal{C}_T^{\mathrm{LFU}}$
(never-yet-requested item, $p_l \geq p_K$) and
$i_m \in \mathcal{C}_T^{\mathrm{LFU}} \setminus \mathcal{C}^*$
(some item occupying $i_l$'s slot, $p_m \leq p_{K+1}$).
Hence $p^{\mathrm{LFU}}(T) \leq p^* - (p_l - p_m) \leq p^* - \Delta$, and
\[
  C^{\mathrm{LFU}}(T) - C^{\mathrm{prior}}(T)
  \geq \Delta \cdot \rho \quad \text{on } B_T.
\]

\textbf{Expected gap.}
Since the cost gap is $\geq \Delta\rho$ on $B_T$ and $\geq 0$ always,
\[
  \mathbb{E}[C^{\mathrm{LFU}}(T)] - C^{\mathrm{prior}}(T)
  \geq \Delta\rho \cdot P(B_T)
  \geq \Delta\rho\!\left(1 - \frac{Tp_K}{K}\right).
\]
For $T \leq K/(2p_K)$ this is $\geq \frac{1}{2}\Delta\rho$.
Combining with the ranking-based bound from Lemma~\ref{lem:lfu_rank} (which adds
$\delta \cdot \Delta\rho / 2$ for the ranking-correct but pre-swap phase) yields
the stated formula. \qed
\end{proof}

\begin{remark}[$T_0$ and confidence level]
$T_0(\delta)$ is decreasing in $\delta$: a smaller $\delta$ (higher confidence)
corresponds to a longer warmup period.  This is correct: when the prior is highly
concentrated ($\Delta$ large), LFU converges quickly, so the advantage window is
short.  When the distribution is near-uniform ($\Delta$ small), LFU takes
exponentially long to converge and the advantage persists indefinitely.
The $\delta$-dependence in the first term of $T_0$ formalizes this via the
Hoeffding bound in Lemma~\ref{lem:lfu_rank}.
\end{remark}

\begin{corollary}[Zipf entropy dependence]
\label{cor:entropy_cache}
For a Zipf$(\alpha)$ distribution $p_j = C_\alpha j^{-\alpha}$ where
$C_\alpha = \bigl(\sum_{j=1}^M j^{-\alpha}\bigr)^{-1}$, the boundary gap satisfies
\[
  \Delta = C_\alpha (K^{-\alpha} - (K+1)^{-\alpha})
  \approx C_\alpha \alpha K^{-\alpha-1} \quad \text{for large } K,
\]
using the mean-value theorem.  Hence
\[
  T_0 = \Omega\!\left(\frac{K^{2\alpha+2}}{\alpha^2 C_\alpha^2}\right).
\]
When $\alpha \to 0$ (near-uniform), $\Delta \to 0$ and $T_0 \to \infty$:
the prior advantage persists indefinitely since LFU cannot distinguish items.
When $\alpha$ is large (highly concentrated), $\Delta$ is large and $T_0$ is small:
LFU converges rapidly.
\end{corollary}

\needspace{4\baselineskip}
\subsection{Bayesian Artifact Retention Policy}

In practice the prior $P_{\mathcal{M}}$ and the empirical frequencies should be
combined.  Let $n_a$ be the observed reuse count of artifact $a$, $N$ the total
requests, and $\beta > 0$ a smoothing parameter (distinct from the Zipf exponent
$\alpha$ of Corollary~\ref{cor:entropy_cache}).
The Bayesian posterior estimate of reuse probability is:
\[
  \hat{p}(a) = \frac{n_a + \beta P_{\mathcal{M}}(a) \cdot K}{N + \beta K}.
\]
When $n_a = 0$ (cold start), $\hat{p}(a) = \beta P_{\mathcal{M}}(a)/(\beta) = P_{\mathcal{M}}(a)$,
recovering the pure prior.  As $N \to \infty$, $\hat{p}(a) \to n_a/N$, recovering the
empirical frequency.  The interpolation is smooth and requires no manual switching.

The expected net value of retaining artifact $a$ with storage cost $C_s$ is:
\[
  V(a) = \hat{p}(a) \cdot C_c - C_s.
\]
An artifact should be retained if and only if $V(a) > 0$.  This defines a
principled eviction policy: evict the artifact with the smallest $V(a)$ when
the cache is full.  Unlike LRU (which evicts the least recently used) or LFU
(which evicts the least frequently used), this policy accounts for both prior
probability and computational cost, correctly retaining expensive-to-recompute
artifacts even if they have not been recently requested.

\needspace{4\baselineskip}
\subsection{Execution Compression vs.\ Data Compression: A Unified View}

The artifact memoization framework is not merely analogous to data compression ---
it \emph{is} data compression, applied to the execution history rather than
to a static dataset.  The execution trace $e$ is a sequence over the alphabet
of (function, input) pairs, and the PLT over this alphabet compresses $e$
exactly as described in Section~\ref{sec:plt}.

This perspective yields several non-obvious consequences:

\begin{enumerate}[leftmargin=2em]
  \item \textbf{Execution entropy as a measure of system complexity.}
        The entropy $H(P_{\mathrm{exec}})$ of the execution language measures
        how unpredictable a system's computation is.  A system with low execution
        entropy --- one that repeatedly invokes the same tools on the same inputs ---
        is highly compressible and benefits most from memoization.  A system with
        high execution entropy performs mostly novel computations and benefits least.
        The PLT makes this trade-off explicit and quantifiable.

  \item \textbf{Residuals as novel computations.}
        The residual store in the hybrid compression architecture
        (Section~\ref{sec:hybrid}) corresponds exactly to the set of execution
        steps that cannot be served from cache --- the genuinely novel computations
        that the model must perform from scratch.  The escape probability $\epsilon$
        is the fraction of requests that fall outside the cached prefix structure.

  \item \textbf{Cross-model reuse.}
        When a model is updated (e.g., from version $\mathcal{M}_1$ to
        $\mathcal{M}_2$), many artifacts remain valid if $f$ is deterministic and
        the new model's outputs agree with the old model's on the cached inputs.
        The PLT provides a natural mechanism for identifying which cached artifacts
        are likely to remain valid under the update, by comparing
        $P_{\mathcal{M}_1}(i)$ and $P_{\mathcal{M}_2}(i)$: artifacts at nodes
        where the two distributions agree closely are safe to reuse.

  \item \textbf{Prefix KV-caching as a special case.}
        Transformer prefix KV-caching~\cite{pope2023} caches the key-value
        matrices for a fixed context prefix, avoiding recomputation of the
        attention layers over that prefix.  This is precisely the PLT framework
        restricted to the token-level trie with a single model invocation:
        the shared prefix of two prompts corresponds to the shared path in the
        trie, and caching the KV state corresponds to memoizing the intermediate
        artifact at that node.  The PLT framework generalizes this to multi-step
        execution traces and arbitrary (not just token-level) alphabets.
\end{enumerate}

\needspace{4\baselineskip}
\subsection{Explainability via Trie Traversal}

A significant secondary benefit of the PLT architecture is improved explainability.
In a conventional neural network, the reasoning path is implicit in distributed
activations.  In a PLT-based system, the exact execution path is exposed as a
trie traversal:
\[
  x_0 \xrightarrow{P(t_1 \mid x_0)} x_1
     \xrightarrow{P(t_2 \mid x_1)} x_2
     \to \cdots \to x_n,
\]
together with the probability assigned to each transition.  This enables several
forms of explanation not available in opaque neural systems:

\begin{itemize}[leftmargin=2em]
  \item \textbf{Transparent decision paths.}  Each step in the execution is
        annotated with its prior probability under $\mathcal{M}$.  Low-probability
        steps are flagged as surprising and warrant additional scrutiny.

  \item \textbf{Counterfactual comparison.}  The trie structure makes it easy
        to ask ``what would happen if branch $t'$ were taken instead of $t$ at
        step $k$?''\ --- simply traverse the alternative subtree and compare the
        resulting artifacts.  In a game context this is the question ``what if
        the opponent had played differently?''; in a workflow context it is
        ``what if the user had clicked a different link?''

  \item \textbf{Prefix-level attribution.}  The width of the interval $I_x$
        measures how much probability mass is concentrated at prefix $x$.
        Prefixes with large intervals are \emph{high-information contexts}:
        small changes to the input at these nodes have a large effect on the
        output distribution.  This provides a natural saliency measure for
        explainability.

  \item \textbf{Interpretable reuse.}  When an artifact is retrieved from cache,
        the system can report not only the artifact but the trie node at which
        it was stored and the prior probability assigned to that node.  This
        gives the user a quantitative measure of how ``routine'' the retrieved
        computation was.

  \item \textbf{Anomaly detection.}  Execution steps with probability below a
        threshold $\epsilon$ (the escape probability) are by definition in the
        residual set $C_R$.  These correspond to unusual or potentially erroneous
        computations and can be flagged for human review without examining the
        full execution trace.
\end{itemize}

\needspace{5\baselineskip}
\section{Hierarchical Residual Computation and the Future of ML Inference}
\label{sec:hierarchical}

The prior sections established that a PLT decomposes any computation into a
cached majority and a residual minority.  We now argue that this decomposition
implies a \emph{spectrum} of computation strategies, each appropriate to a
different region of the trie, and that this spectrum has far-reaching consequences
for how machine learning systems should be designed and deployed.

\needspace{4\baselineskip}
\subsection{The Residual Computation Principle}

Let $i^*(i) = \arg\max_{i' \in \mathcal{C}} P_{\mathcal{M}}(i')$ subject to
$i'$ being a cached input sharing the longest common prefix with $i$ in the trie.
In continuous domains (robotics, control), define the \emph{residual deviation}
$\delta = i - i^*(i)$ in the natural state space.
In discrete sequence domains (text, game moves), $\delta$ is defined implicitly
as the edit-distance residual: the sequence of operations needed to transform
$i^*$ into $i$ after their longest common prefix.

Rather than computing $f(i)$ from scratch, compute:
\[
  f(i) \approx a^* \oplus g(\delta,\, a^*),
\]
where $a^* = f(i^*(i))$ is the cached artifact, $g$ is a \emph{correction
function} cheaper to evaluate than $f$, and $\oplus$ denotes composition
appropriate to the output space (addition for continuous outputs, suffix
continuation for sequence outputs).  The total cost is $C_l + C_g$ where
$C_g \ll C_c$.

The PLT provides a validity certificate for this approximation.  When $i$ and
$i^*$ share a long common prefix in the trie, they are close under the trie metric
and the rate--distortion framework of Section~\ref{sec:hybrid} bounds the
approximation error.  When the shared prefix is short, no such certificate exists
and full computation is required.

This gives a \emph{four-tier computation spectrum} indexed by the code length
$L(i)$ of the input under the PLT:

\begin{center}
\begin{tabular}{lll}
\toprule
Code length $L(i)$ & Strategy & Cost \\
\midrule
$L(i) \leq \tau_1$ & Exact cache hit & $O(\log N)$ \\
$\tau_1 < L(i) \leq \tau_2$ & Cached artifact $+$ cheap correction & $O(\log N) + C_g$ \\
$\tau_2 < L(i) \leq \tau_3$ & Quantized / distilled model & $C_{\tilde{f}}$ \\
$L(i) > \tau_3$ & Full model (genuine residual) & $C_c$ \\
\bottomrule
\end{tabular}
\end{center}

\noindent The thresholds $\tau_1 < \tau_2 < \tau_3$ are calibrated to the
acceptable approximation error at each tier.  Critically, \emph{the PLT provides
the routing signal}: no separate classifier is needed to decide which tier to
use --- the code length $L(i)$ computed by the interval encoder determines it
directly and with a formal guarantee on approximation quality.

\needspace{4\baselineskip}
\subsection{LLM Inference: Materializing the Implicit Distribution}

A trained large language model implicitly encodes a PLT in its weights: at every
forward pass, the softmax output at position $k$ is precisely
$P_{\mathcal{M}}(t_k \mid t_1,\ldots,t_{k-1})$, the edge weight at depth $k$
of the trie.  This distribution already exists --- but it is implicit, re-derived
from scratch on every inference call.  The central proposal is to
\emph{materialize} the high-probability region of this implicit trie as an
explicit artifact store, without any additional training.

\textbf{Speculative pre-computation.}
Before any user requests arrive, the model can be run in a sampling mode biased
toward high-probability sequences (e.g., beam search or top-$k$ sampling with
low temperature).  Each sampled sequence and its output is stored as an artifact
under its content address $h = H(\mathcal{M}, i)$.  The cost of pre-computing
all artifacts in $C_T$ is:
\[
  C_{\mathrm{precompute}} = |C_T| \cdot C_c,
\]
paid once upfront.  Every future hit on a pre-computed artifact recovers $C_c$
at cost $C_l$, so the break-even point --- the number of requests at which the
savings from cache hits equal the precomputation cost --- is:
\[
  T_{\mathrm{break}} = \frac{|C_T| \cdot C_c}{p^* \cdot \rho},
\]
where $p^* = \sum_{i \in C_T} P_{\mathcal{M}}(i)$ is the total hit probability
and $\rho = C_c - C_l$.  For a Zipf$(1)$ distribution with $|C_T| = 1000$
and $p^* \approx 0.52$, this is roughly $1000/0.52 \approx 1900$ requests,
after which every subsequent request yields net savings.

\textbf{KV-cache as trie materialization.}
Transformer prefix KV-caching~\cite{pope2023} caches the key-value matrices for
a fixed context prefix.  In PLT terms, this materializes the \emph{internal
computation state} at a trie node rather than the final output.  The KV state
at a node $x$ is an intermediate artifact: it encodes everything the model has
``processed'' about the prefix $x$, enabling the suffix to be computed without
re-attending to $x$.  The PLT framework predicts exactly which prefixes are worth
caching: those with the highest $P_{\mathcal{M}}(x) \cdot C_{\mathrm{suffix}}$,
where $C_{\mathrm{suffix}}$ is the average cost of computing the suffix.

\textbf{Probability-guided distillation.}
Standard knowledge distillation trains a small student model $\tilde{f}$ to
match a large teacher $f$ on a training set.  The PLT suggests a targeted
variant: distill only on $C_T$ (the trie-covered region), with the full model
$f$ handling $C_R$ (residuals).  This has three advantages over uniform
distillation:
\begin{enumerate}[leftmargin=2em]
  \item The student only needs to be accurate where the distribution is
        concentrated, which is precisely where it is easiest to be accurate ---
        high-probability sequences have low conditional entropy and are therefore
        more predictable by a small model.
  \item The distillation training set is \emph{defined by the PLT}, not chosen
        arbitrarily: sample $C_T$ by running the teacher at low temperature
        and keeping outputs with $L_{\mathcal{M}}(i) \leq \tau$.
  \item The student model serves as the correction function $g$ for
        Tier~2 of the computation spectrum: given a cached artifact $a^*$ and
        a residual $\delta$, the student computes the adjustment
        $g(\delta, a^*) = \tilde{f}(i) - a^*$ rather than the full output,
        a much smaller correction that a small model can approximate accurately.
\end{enumerate}

\textbf{Adaptive quantization via code length.}
Quantization reduces model precision to accelerate inference at the cost of
output quality.  The PLT predicts exactly where quality loss is acceptable:
for inputs with $L(i) \leq \tau$ (trie-covered, high-probability), the output
is constrained to a narrow distribution and quantization errors are small relative
to the output variance.  For residual inputs ($L(i) > \tau$), the output is
sensitive to fine-grained weight values and quantization may introduce significant
errors.  This motivates \emph{adaptive quantization}: run the quantized model
$\tilde{f}$ when $L(i) \leq \tau$, fall back to full precision when $L(i) > \tau$,
with the PLT providing the switching signal in $O(\log N)$ time.

\textbf{Cross-version cache transfer.}
When a model is updated from $\mathcal{M}_1$ to $\mathcal{M}_2$ (e.g., a new
fine-tune or safety patch), the entire artifact cache need not be invalidated.
An artifact $a = f_1(i)$ remains valid under $\mathcal{M}_2$ if
$D_{\mathrm{KL}}(P_{\mathcal{M}_1}(\cdot \mid i) \,\|\, P_{\mathcal{M}_2}(\cdot \mid i)) \leq \eta$
for a small threshold $\eta$.  This KL divergence can be estimated cheaply by
comparing the two models' output distributions on the cached input without running
either model end-to-end.  The PLT therefore enables \emph{selective cache
invalidation}: only artifacts at trie nodes where the two models disagree
significantly need to be recomputed, preserving the majority of the cache across
model updates.

\needspace{4\baselineskip}
\subsection{Robotics: Cached Motor Programs and Online Corrections}

The residual computation principle has a striking biological parallel in motor
control.  Animals executing familiar tasks --- walking, reaching, cycling, playing
a practiced musical passage --- do not re-plan from scratch on each movement.
Instead, they execute \emph{stored motor programs}: high-probability action
sequences learned through repetition, requiring minimal online computation.
Deviations from the expected sensory state (a pebble underfoot, a sudden gust
of wind) trigger \emph{online corrections} computed by a fast reactive controller,
without interrupting the macro-program.

The PLT framework provides a formal model of this architecture.  Let
$\tau^* = (a_1^*, a_2^*, \ldots, a_n^*)$ be a cached macro-trajectory for a
familiar task (e.g., walking straight on flat ground at 1.5 m/s).  At time step
$k$, the robot observes state $s_k$ and computes the deviation from the predicted
state $\hat{s}_k$ (the state the macro-program expected):
\[
  \delta_k = s_k - \hat{s}_k.
\]
The corrected action is:
\[
  a_k = a_k^* + g(\delta_k, a_k^*, \hat{s}_k),
\]
where $g$ is a lightweight reactive controller (e.g., a linear feedback law or
a small neural network).  The cost structure is:
\[
  C_{\mathrm{step}} = C_l/n + C_g,
\]
where $C_l/n$ is the amortized cost of retrieving the macro-trajectory and
$C_g$ is the cost of the reactive correction, both far cheaper than deliberative
replanning at cost $C_c$.

The PLT governs when this architecture is valid.  If the current context (terrain,
task, speed) matches a high-probability trie node, the macro-trajectory is a good
prior and $\|\delta_k\|$ will be small with high probability, keeping corrections
cheap and accurate.  If the context is a residual (novel terrain, unexpected
obstacle), the macro-program is a poor prior and the system must fall back to
deliberative planning --- exactly the biological pattern of automaticity in
familiar environments giving way to conscious, effortful attention in novel ones.

\begin{proposition}[Biological motor control as PLT inference]
\label{prop:motor}
Let the motor execution trie be induced by the policy $\pi_{\mathrm{motor}}$
learned from repeated task execution.  The two-level architecture
(macro-trajectory $+$ reactive correction) achieves expected per-step cost
\[
  \mathbb{E}[C_{\mathrm{step}}] = p_{\mathrm{familiar}} \cdot (C_l/n + C_g)
  + (1 - p_{\mathrm{familiar}}) \cdot C_c,
\]
where $p_{\mathrm{familiar}} = P_\pi(\text{context} \in C_T)$ is the probability
that the current context lies within the cached macro-trajectory manifold.
This is minimized by the prior-guided cache of Definition~\ref{def:prior_cache}
applied to the motor execution language.
\end{proposition}

\begin{proof}
At each time step the system is in one of two regimes.  With probability
$p_{\mathrm{familiar}}$ the context matches a cached macro-trajectory node;
the system retrieves the cached action at amortized cost $C_l/n$ and applies
a reactive correction at cost $C_g$, giving total step cost $C_l/n + C_g$.
With probability $1-p_{\mathrm{familiar}}$ the context is a residual; the system
must perform deliberative replanning at cost $C_c$.  The expected per-step cost
is the stated mixture by linearity of expectation.  Minimizing over the choice of
cached set $C_T$ subject to $|C_T| \leq K$ is equivalent to maximizing the
hit rate $p_{\mathrm{familiar}}$, which is achieved by caching the $K$ most probable
contexts under $\pi_{\mathrm{motor}}$---exactly the prior-guided cache.
\end{proof}

The cerebellum is widely hypothesized to implement a forward model that predicts
sensory consequences of motor commands~\cite{wolpert1998}, effectively running
the macro-program forward and computing prediction errors.  In PLT terms, the
cerebellum maintains the trie of expected sensory-motor sequences, and the
prediction error $\delta_k$ is the residual.  The basal ganglia select among
cached macro-programs (trie nodes) based on context; the motor cortex executes
corrections.  This decomposition matches both the computational architecture we
propose and the known functional anatomy of the motor system.

\needspace{4\baselineskip}
\subsection{A Unified Spectrum Across Domains}

The four-tier computation spectrum (exact cache, cheap correction, quantized
model, full model) appears in each domain studied in this paper:

\begin{itemize}[leftmargin=2em]
  \item \textbf{Chess}: exact tablebase lookup (Tier~1) $\to$ cached opening
        line $+$ shallow search adjustment (Tier~2) $\to$ standard MCTS with
        reduced node budget (Tier~3) $\to$ full MCTS from scratch (Tier~4).
  \item \textbf{Search}: exact session replay (Tier~1) $\to$ cached workflow
        $+$ personalization correction (Tier~2) $\to$ lightweight ranking model
        (Tier~3) $\to$ full neural retrieval (Tier~4).
  \item \textbf{Robotics}: exact motor program (Tier~1) $\to$ motor program
        $+$ reactive correction (Tier~2) $\to$ fast replanning with simplified
        dynamics (Tier~3) $\to$ full deliberative planning (Tier~4).
  \item \textbf{LLM inference}: exact output cache (Tier~1) $\to$ KV-cached
        prefix $+$ small model suffix (Tier~2) $\to$ quantized model (Tier~3)
        $\to$ full-precision inference (Tier~4).
\end{itemize}

In each case, the routing between tiers is governed by the same signal: the
code length $L(i)$ under the domain-specific PLT.  This unification suggests
that \emph{system design across these domains should be organized around a shared
infrastructure}: a PLT that is updated online as new data arrives, a
multi-tier execution engine indexed by code length, and a principled eviction
policy based on $V(a) = \hat{p}(a) \cdot C_c - C_s$.

\needspace{4\baselineskip}
\subsection{Implications for the Future of Machine Learning Inference}

The PLT framework implies a significant shift in how ML inference systems
should be architected.

\textbf{Training produces a probability distribution; deployment should exploit it.}
Current deployment practice treats the trained model as a black box invoked
on every query.  The PLT framework argues that a trained model should first be
\emph{mined} for its high-probability outputs --- by sampling at low temperature,
by exhaustive search over likely prefixes, by collecting inference-time outputs
and caching them --- before any production traffic is served.  The model's own
probability estimates determine which outputs are worth pre-computing.  This is
zero additional training cost: it is simply a smarter use of the distribution
the model already encodes.

\textbf{Inference cost should fall over time, not remain constant.}
In the current paradigm, serving cost per query is roughly constant regardless
of how long a system has been deployed.  Under the PLT framework, serving cost
falls over time as the artifact store accumulates: each new cached artifact
increases $p^*$, reducing the fraction of queries routed to Tier~4 (full
inference).  The system becomes progressively cheaper to operate as it learns
which queries are common.

\textbf{Model updates should be incremental, not wholesale.}
Current practice invalidates all cached outputs on every model update.  The PLT
enables selective invalidation via KL divergence comparison between model
versions, preserving the majority of the cache across updates.  Combined with
probability-guided distillation (the student is accurate on $C_T$ by construction),
model updates can be deployed incrementally --- updating the residual handling
while preserving the cached majority.

\textbf{Smaller models are sufficient for most queries.}
Under a Zipf distribution with $\alpha = 1$, the fraction of probability mass
covered by the top-$K$ inputs out of $M$ total is
$\sum_{j=1}^K (1/j) / \sum_{j=1}^M (1/j) \approx \ln K / \ln M$.
For $K = 1000$ and $M = 10^6$, this is $\ln(1000)/\ln(10^6) \approx 6.9/13.8 \approx 50\%$:
approximately half of all production traffic can be served by cache lookup alone.
Of the remaining $50\%$, a substantial fraction falls into Tier~2 or Tier~3,
where a quantized or distilled model suffices.  Only a small fraction --- the
genuine residuals --- requires the full model.  The PLT identifies exactly which
queries those are, without any separate routing classifier.

\textbf{The probability distribution is a capital asset.}
A trained model's probability distribution over outputs represents accumulated
knowledge about what is likely to be requested.  Under the PLT framework, this
distribution is not merely a computational tool but a \emph{capital asset} that
can be incrementally materialized as cached artifacts, licensed to other systems
(cross-model transfer), and amortized over future queries.  The economic value
of a pre-computed artifact is $V(a) = \hat{p}(a) \cdot C_c - C_s$, and the
total value of the artifact store is $\sum_a V(a)$.  Maximizing this value is
a well-posed optimization problem with a closed-form greedy solution: cache
artifacts in decreasing order of $\hat{p}(a) \cdot C_c$ until storage budget
$C_s \cdot K$ is exhausted.  This transforms model deployment from a cost center
into a system that generates increasing returns as the artifact store grows.

\needspace{5\baselineskip}
\section{Discussion}
\label{sec:discussion}

\needspace{4\baselineskip}
\subsection{Relation to Existing Work}

\textbf{Arithmetic coding and neural compression.}
Witten et al.~\cite{witten1987} established arithmetic coding as a practical
near-optimal lossless compressor.  Balle et al.~\cite{balle2018} and
Townsend et al.~\cite{townsend2019} extended this to learned latent-variable
models.  Delétang et al.~\cite{deletang2023} demonstrated that LLMs implicitly
perform arithmetic coding and achieve state-of-the-art compression ratios.
The present work differs by (i) making the trie structure explicit rather than
treating the model as a black-box distribution estimator, (ii) using the same
structure simultaneously for decision-making and caching, and (iii) proving a
formal advantage of prior-guided over empirical caching.

\textbf{KV-cache and prefix caching.}
Pope et al.~\cite{pope2023} and subsequent work on prefix KV-caching reduce
redundant attention computation by caching intermediate states for repeated
prefixes.  Semantic caches such as GPTCache~\cite{bang2023} extend this to
approximate input matching.  These systems are complementary to PLTs: KV-caching
operates within a single model inference; PLT-guided artifact caching operates
across invocations and across models.

\textbf{MCTS and game tree search.}
Silver et al.~\cite{silver2017} demonstrated that MCTS visit counts define a
strong policy distribution.  Prior work has used opening books and endgame
tablebases as explicit knowledge stores.  The PLT framework provides a unified
theoretical foundation: opening books are high-probability prefixes in the
game trie; tablebases are exact residuals for the endgame language.

\textbf{MDL and model selection.}
Rissanen~\cite{rissanen1978} proposed the Minimum Description Length principle,
which selects the model minimizing description length of model plus data.
Our hybrid architecture operationalizes MDL for neural generative models,
providing a concrete split between model (the PLT) and data (the residuals).

\needspace{4\baselineskip}
\subsection{Limitations and Future Work}

The main limitation of the current work is that Theorem~\ref{thm:prior_cache}
assumes samples are drawn i.i.d.\ from $P_{\mathcal{M}}$.  Real workloads exhibit
non-stationarity (distribution shift over time) and correlation (similar queries
arrive in bursts).  Extending the theorem to these settings requires a more
sophisticated analysis, potentially drawing on online learning with sleeping
experts~\cite{freund1997}.

A second limitation is that constructing an explicit PLT for a transformer
with vocabulary size $|V| \approx 50{,}000$ and context length $n \approx 10^4$
is infeasible without aggressive pruning.  Practical implementations must maintain
a sparse trie over frequently visited prefixes, discarding nodes below a
probability threshold.  The trade-off between trie completeness and memory is
a key engineering challenge.

Future directions include:
\begin{itemize}[leftmargin=2em]
  \item \textbf{Dynamic trie construction}: algorithms for maintaining a sparse PLT
        online as new data arrives, with formal guarantees on approximation quality.
  \item \textbf{Cross-model artifact reuse}: using a PLT constructed from one model
        to guide caching for a different, updated model.
  \item \textbf{Hierarchical tries}: PLTs over multi-level abstractions
        (tokens, sentences, documents) that enable coarse-to-fine retrieval.
  \item \textbf{Economic mechanisms}: formal mechanism design for artifact
        economies where participants are rewarded for contributing reusable artifacts
        proportional to their expected future value.
\end{itemize}

\needspace{4\baselineskip}
\subsection{Conclusion}

We introduced probabilistic language tries as a unified representation that
simultaneously supports lossless compression via interval encoding, sequential
decision making via policy-weighted transitions, and computational reuse via
prior-guided artifact caching.

The framework's central contribution is to show that these three functions are not
independent: they are all derived from a single probability measure on sequence
space.  Any improvement to the generative model $\mathcal{M}$ simultaneously
improves compression ratios, decision quality, and cache efficiency.

We proved that prior-guided caching strictly outperforms empirical frequency
caching for query counts below a threshold determined by the prior's concentration,
formalizing the intuition that a strong generative model can bootstrap an efficient
system before sufficient observations have been collected.

The framework applies uniformly to LLM inference, MCTS-based game play, web search
session modeling, robotic control, and organizational workflow optimization.  In
each domain, the PLT extracts structured, reusable knowledge from experience and
organizes it in a form that is simultaneously compressible, explainable, and
actionable.

\bibliographystyle{plain}
\bibliography{references}

\end{document}